\documentclass[conference]{IEEEtran}
\IEEEoverridecommandlockouts 

\usepackage{verbatim} 
\usepackage{graphicx, caption}
\usepackage{amssymb}
\usepackage{amsmath}
\usepackage{amsthm}
\usepackage{mathtools}
\usepackage{cite}
\usepackage{stfloats}
\usepackage{subfigure}
\usepackage{epstopdf}
\usepackage{psfrag}
\usepackage[mathscr]{euscript}
\usepackage{acronym}  
\usepackage{booktabs} 
\usepackage[linesnumbered,ruled,vlined]{algorithm2e}
\usepackage[table]{xcolor}

\SetKwInput{KwInput}{Input}
\SetKwInput{KwOutput}{Output}

\usepackage{color}
\usepackage{dsfont}
\usepackage{bbm}

\acrodef{IoT}{Internet of Things}
\acrodef{BS}{base station}
\acrodef{pdf}{probability density function}
\acrodef{i.i.d.}{independent and identically distributed}
\acrodef{CDF}{cumulative distribution function}
\acrodef{FL}{federated learning}
\acrodef{ML}{machine learning}
\acrodef{SGD}{stochastic gradient descent}
\acrodef{MAC}{multiply-accumulate}
\acrodef{CNN}{convolutional neural network}
\acrodef{DNN}{deep neural network}
\acrodef{QNN}{quantized neural network}
\acrodef{OFDMA}{orthogonal frequency domain multiple access}

\usepackage{color}
\usepackage{dsfont}
\usepackage{bbm}

\newtheorem{theorem}{Theorem}
\newtheorem{lemma}{Lemma}
\newtheorem{prop}{Proposition}

\newcommand{\E}{\mathbb{E}}
\renewcommand{\P}{\mathbb{P}}

\newcommand{\Whole}{N}
\newcommand{\weights}{\boldsymbol{w}}
\newcommand{\qweights}{\boldsymbol{w}^{Q, k}}
\newcommand{\inputvec}{\boldsymbol{x}_{kl}}
\newcommand{\datasize}{D_k}
\newcommand{\loss}{F_k}
\newcommand{\globall}{F}
\newcommand{\globalit}{t}
\newcommand{\scheduleset}{\mathcal{N}_t}
\newcommand{\schedulesize}{K}
\newcommand{\SGDrun}{I}
\newcommand{\precision}{n}
\newcommand{\smallest}{\kappa}
\newcommand{\Quantize}{Q}
\newcommand{\ka}{\nonumber \\}
\newcommand{\learningrate}{\eta}
\newcommand{\minibatch}{\xi}
\newcommand{\modelupdate}{\boldsymbol{d}}
\newcommand{\qmodelupdate}{\boldsymbol{d}^{Q,k}}
\newcommand{\smooth}{L}
\newcommand{\strong}{\mu}
\newcommand{\sgdvariance}{\sigma}
\newcommand{\sgdbound}{G}
\newcommand{\constant}{v}

\newcommand{\MACunits}{p}
\newcommand{\MACE}{E_\text{MAC}}
\newcommand{\localb}{E_\text{l}}
\newcommand{\mainb}{E_\text{m}}
\newcommand{\MACamplitude}{A}
\newcommand{\premax}{n_{\text{max}}}
\newcommand{\MACex}{\alpha}
\newcommand{\localE}{E^{C, k} (n)}
\newcommand{\tE}{E^{UL, k} (n)}
\newcommand{\numweights}{\text{$N_s$}}
\newcommand{\numMAC}{\text{$N_c$}}
\newcommand{\numout}{\text{$O_s$}}
\newcommand{\achievalble}{r}
\newcommand{\accuracy}{\epsilon}
\newcommand{\obj}{f_E (\precision)}

\allowdisplaybreaks 

\begin{document}
	\newcommand{\paperTitle}{Title}
\vspace{-0.0cm}


\title{\vspace{-0.0cm} On the Tradeoff between Energy, Precision, and Accuracy in Federated Quantized Neural Networks}

\author{ \vspace{-0.05 cm}	
	
\IEEEauthorblockN{
	Minsu~Kim$^1$,  Walid~Saad$^1$, Mohammad~Mozaffari$^2$, and Merouane~Debbah$^{3, 4}$
}\\[-0.5em]
\vspace{-2.0mm}
\IEEEauthorblockA{
$^1$ Wireless@VT, Bradley Department of Electrical and Computer Engineering, Virginia Tech, Blacksburg, VA, USA. 
\\$^2$ Ericsson Research, Santa Clara, CA, USA.
	\\ $^3$ Technology Innovation Institute, Abu Dhabi, United Arab Emirates.
	\\ $4$ Mohamed Bin Zayed University of Artificial Intelligence, Abu Dhabi, United Arab Emirates.
	\\Emails: \{msukim, walids\}@vt.edu, mohammad.mozaffari@ericsson.com,  merouane.debbah@tii.ae.
	\vspace{-0.8 cm}	
}
\thanks{This work was supported by the U.S. National Science Foundation under Grants  CNS-1814477 and CNS-2114267.
\vspace{-1.0mm}
}
}

\maketitle 

%

%

\acresetall
\begin{abstract}
Deploying federated learning (FL) over wireless networks with resource-constrained devices requires balancing between accuracy, energy efficiency, and precision. Prior art on FL often requires devices to train deep neural networks (DNNs) using a 32-bit precision level for data representation to improve accuracy. However, such algorithms are impractical for resource-constrained devices since DNNs could require execution of millions of operations. Thus, training DNNs with a high precision level incurs a high energy cost for FL. In this paper, a quantized FL framework, that represents data with a finite level of precision in both local training and uplink transmission, is proposed. Here, the finite level of precision is captured through the use of quantized neural networks (QNNs) that quantize weights and activations in fixed-precision format. In the considered FL model, each device trains its QNN and transmits a quantized training result to the base station. Energy models for the local training and the transmission with the quantization are rigorously derived. An energy minimization problem is formulated with respect to the level of precision while ensuring convergence. To solve the problem, we first analytically derive the FL convergence rate and use a line search method. Simulation results show that our FL framework can reduce energy consumption by up to 53\% compared to a standard FL model. The results also shed light on the tradeoff between precision, energy, and accuracy in FL over wireless networks.  
\end{abstract}

\section{Introduction}
The emergence of \ac{FL} ushered in a new era of distributed inference that can alleviate  data privacy concerns \cite{Ti:20}. In \ac{FL}, massively distributed mobile devices and a central server (e.g., a \ac{BS}) collaboratively train a shared model without requiring devices to share raw data. Many \ac{FL} algorithms employ complex \acp{DNN} to achieve a high accuracy by allocating many bits for the precision level in data representation \cite{HB:16}. DNN structures, such as \acp{CNN}, can have tens of millions of parameters and billions of \ac{MAC} operations \cite{Ro:19}. In practice, the energy consumed for computation and memory access is proportional to the level of precision \cite{MB:17}. Hence, computationally intensive neural networks with a conventional 32 bits full precision level may not be suitable for deployment on energy-constrained mobile and \ac{IoT} devices. In addition, a \ac{DNN} may increase the energy consumption of transmitting a training result due to the large model size. To design an energy-efficient \ac{FL} scheme, one could reduce the level of precision to decrease the energy consumption for the computation and transmission. However, the reduced precision level could introduce quantization error that degrades the accuracy and the convergence rate  of \ac{FL}. Therefore, deploying real-world FL frameworks over wireless systems requires one to \emph{balance precision, accuracy, and energy efficiency} -- a major challenge facing future distributed learning frameworks.

Remarkably, despite the surge in research on the use of \ac{FL}, only a handful of works in \cite{Stesa:21, Ngutra:19, Zhaya:21, Khaung:21, SZ:21, Yuyan:21} have studied the energy efficiency of \ac{FL} from a system-level perspective. A novel analytical framework that derived energy efficiency of \ac{FL} algorithms in terms of the carbon footprint was proposed in \cite{Stesa:21}. Meanwhile, in \cite{Ngutra:19}, the authors formulated an energy minimization problem under heterogeneous power constraints of mobile devices. The work in \cite{Zhaya:21} investigated a resource allocation problem to minimize the total energy consumption considering the convergence rate. In \cite{Khaung:21}, the energy consumption of \ac{FL} was minimized by controlling workloads of each device, which has heterogeneous computing resources. The work in \cite{SZ:21} proposed a quantization scheme for both uplink and downlink transmission in \ac{FL} and analyzed the impact of the quantization on the convergence rate. The authors in \cite{Yuyan:21} considered a novel \ac{FL} setting, in which each device trains a binary neural network so as to improve the energy efficiency of transmission by uploading the binary parameters to the server.

However, the works in \cite{Stesa:21, Ngutra:19, Zhaya:21, Khaung:21, SZ:21} did not consider the energy efficiency of their \ac{DNN} structure during training. Since devices have limited computing and memory resources, deploying an energy-efficient \ac{DNN} will be a more appropriate way to reduce the energy consumption of \ac{FL}. 
Although the work in [11] considered binarized neural networks during training, this work did not optimize the quantization levels of the neural network to balance the tradeoff between precision and energy. To the best of our knowledge, there is no work that jointly considers the tradeoff between precision, energy, and accuracy.

The main contribution of this paper is a novel energy-efficient quantized \ac{FL} framework that can represent data with a finite level of precision in both local training and uplink transmission. In our \ac{FL} model, each device trains a \ac{QNN}, whose weights and activations are quantized with a finite level of precision, so as to decrease energy consumption for computation and memory access. After training, each device quantizes the result with the same level of precision used in the local training and transmits it to the \ac{BS}. The \ac{BS} aggregates the received information to generate a new global model and broadcasts it back to the devices. To quantify the energy consumption, we propose rigorous energy model for the local training based on the physical structure of a processing chip. We also derive the energy model for the uplink transmission considering the quantization. To achieve a high accuracy, \ac{FL} requires a high level of precision at the cost of increased total energy consumption. Meanwhile, although, a low level of precision can decrease the energy consumption per iteration, it will decrease the convergence rate to achieve a target accuracy. Thus, there is a need for a new approach to analyze and optimize the tradeoff between precision, energy, and accuracy. To this end, we formulate an optimization problem by controlling the level of precision to minimize the total energy consumption while ensuring convergence with a target accuracy. To solve the problem, we first analytically derive the convergence rate of our \ac{FL} framework and use a line search method to numerically find the local optimal solution. Simulation results show that our \ac{FL} model can reduce the energy consumption up to 53\% compared to a standard \ac{FL} model, which uses 32-bit full-precision for data representation. The results also shed light on the tradeoff between precision, energy efficiency, and accuracy in \ac{FL} over wireless networks.

The rest of this paper is organized as follows. Section \ref{sec:system model} presents the system model. In Section \ref{sec:problem formulation}, we describe the studied problem. Section \ref{sec:simlulation results} provides simulation results. Finally, conclusions are drawn in Section \ref{sec:conclusion}.

\begin{figure}
			\includegraphics[width=0.9\columnwidth]{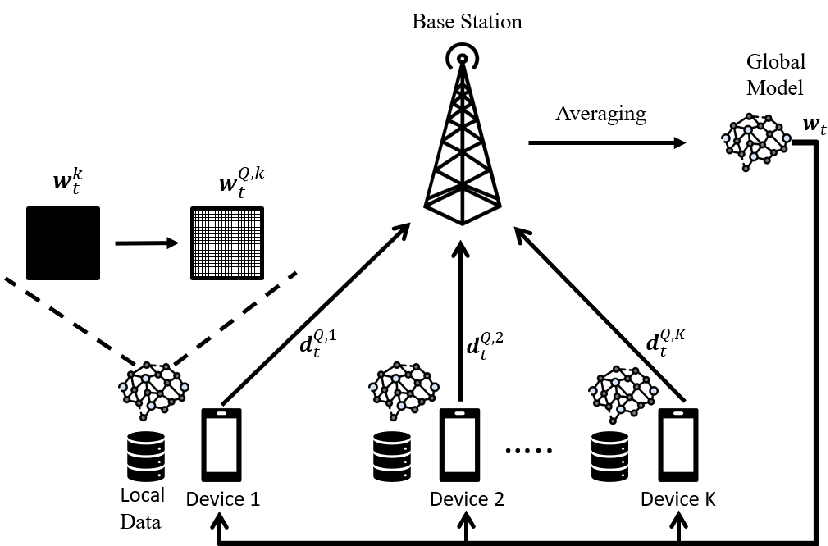}
 			\captionsetup {singlelinecheck = false}

		\caption{An illustration of the quantized \ac{FL} model over wireless network.}
		\label{fig:system_model}

\vspace{-0.6 cm}	
\end{figure}

\section{System Model} \label{sec:system model}
We consider an \ac{FL} system, in which $\Whole$ devices (e.g. edge or mobile devices) are connected to one \ac{BS}. As shown in Fig. \ref{fig:system_model}, the \ac{BS} and devices collaboratively perform an \ac{FL} algorithm for executing a certain data analysis task. Each device $k$ has local dataset $\mathcal{D}_k = \{\inputvec, y_{kl}\}$, where  $\l = 1, \dots, \datasize$. In particular, $\{\inputvec, y_{kl}\}$ is an input-output pair for image classification, where $\boldsymbol{x}_{kl}$ is an input vector and $y_{kl}$ is the corresponding output. We define a loss function $f(\weights,\inputvec, y_{kl} )$ to quantify the performance of a \ac{ML} model with parameters $\weights \in \mathbb{R}^d$ over $\{\inputvec, y_{kl}\}$. Since device $k$ has $\datasize$ data samples, its local loss function is given by 
\begin{align}
 	\loss(\weights) = \frac{1}{\datasize} \sum_{l=1}^{\datasize}  f(\weights,\inputvec, y_{kl} ).
\end{align}
We define the global loss function over $\Whole$ devices as follows:
\begin{align}
	\globall(\weights) = \sum_{k=1}^{\Whole} \frac{\datasize}{D} \loss (\weights) = \frac{1}{D} \sum_{k=1}^{N} \sum_{l=1}^{\datasize} f(\weights,\inputvec, y_{kl} ), 
\end{align}
where $D = \sum_{k=1}^{\Whole} \datasize$ is the total size of the entire dataset. The \ac{FL} process aims to find the optimal model parameters $\weights$ that can minimize the global loss function as follows 
\begin{align}
	\min_{\weights} \globall (\weights). \label{FL problem}
\end{align}

Solving problem \eqref{FL problem} typically requires an iterative process between the \ac{BS} and devices. However, in practical systems, such as an \ac{IoT}, the devices are energy-constrained. They are unable to run a power consuming \ac{FL} process. Hence, we propose to manage the level of precision of our \ac{FL} to reduce the energy consumption for computation, memory access, and transmission. As such, we adopt a \ac{QNN} structure whose weights and activations are quantized in fixed-point format rather than conventional 32-bit floating-point format \cite{IH:16}. 
\subsection{Quantized Neural Networks}
In our model, each device trains a \ac{QNN} of identical structure using $\precision$ bits of precision for quantization. We can express data more precisely if we increase $\precision$ at the cost of more energy usage. We can represent any given number in fixed-point format such as $[\Omega. \omega]$, where $\Omega$ is the integer part and $\omega$ is the fractional part of the given number \cite{SG:15}. Here, we use one bit to represent the integer part and $(\precision-1)$ bits for the fractional part. Then, the smallest positive number we can present would be $\kappa = 2^{-n+1}$, and the possible range of numbers with $\precision$ bits will be $[-1 ,1-2^{-n+1}]$. Note that a \ac{QNN} restricts the value of weights to [-1, 1]. We consider a stochastic quantization scheme \cite{SG:15}, where any given $w \in \weights$ is quantized as follows:
\begin{align}
	\Quantize(w)  = 
	\begin{cases}
		\lfloor w \rfloor, & \quad  \text{with probability} \quad \frac{\lfloor x \rfloor +\smallest - w}{\smallest},\\
		\lfloor w \rfloor + \smallest,              & \quad \text{with probability} \quad \frac{w - \lfloor w \rfloor}{\smallest},
	\end{cases}
\end{align}    	
where $\lfloor w \rfloor$ is the largest integer multiple of $\smallest$ less than or equal to $w$.

We denote the quantized weights of layer $l$ as $\qweights_{(l)} = Q(\weights^k_{(l)})$ for device $k$. Then, the outputs of layer $l$ will be: 
\begin{align}
	o_{(l)} = g_{(l)} (\qweights_{(l)}, o^{\Quantize}_{(l-1)}),
\end{align}
where $g(\cdot)$ is the operation of layer $l$ on the input, such as activation and batch normalization, and  $o^{\Quantize}_{(l-1)}$ is the quantized output from the previous layer $l-1$. Note that the output $o_{(l)}$ will be quantized and fed into the next layer as an input. For training, we use \ac{SGD} algorithm as follows
\begin{align}
	\weights^k \leftarrow \weights^k - \learningrate \nabla \loss(\qweights, \minibatch^k), \label{SGD}
\end{align}
where $\learningrate$ is the learning rate and $\minibatch$ is a mini-batch for the current update. Then, we restrict the values of $\weights^k$ to $[-1, 1]$ as $\weights^k \leftarrow \text{clip}(\weights^k, -1, 1),$
where $\text{clip}(\cdot, -1, 1)$ projects each input to 1 (-1) for any input larger (smaller) than 1 (-1), or returns the same value as the input. Otherwise, $\weights^k$ can become very large without a meaningful impact on quantization \cite{IH:16}. After each training, $\weights^k$ are quantized as $\qweights$.  \vspace{-1.0mm}

\subsection{\ac{FL} model}

For learning, without loss of generality, we adopt FedAvg \cite{HB:16} to solve problem \eqref{FL problem}. At each global iteration $\globalit$, the \ac{BS} randomly selects a set of devices $\scheduleset$ with $|\scheduleset| = K$ and broadcasts the current global model $\weights_t$ to the scheduled devices. Each device in $\scheduleset$ trains its local model based on the received global model by running $\SGDrun$ steps of \ac{SGD} on its local loss function as below
\begin{align}
	\weights_{t, \tau}^k \hspace{-0.5mm}= \hspace{-0.5mm} \weights_{t, \tau-1}^k  \hspace{-0.5mm}- \hspace{-0.5mm} \learningrate_t \nabla \loss(\qweights_{t, \tau-1}, \minibatch^k_\tau), \forall \tau \hspace{-0.5mm} = \hspace{-0.5mm} 1, \dots , \SGDrun , \label{SGD_update}
\end{align}	
where $\learningrate_t$ is the learning rate at global iteration $t$.
Note that unscheduled devices do not perform local training. Then, $\schedulesize$ devices calculates the model update $\modelupdate^k_{t+1} = \weights^k_{t+1} - \weights^k_t$, where $\weights^k_{t+1} = \weights^k_{t, \SGDrun}$ and $\weights^k_t = \weights^k_{t, 0}$ \cite{SZ:21}. Typically, $\modelupdate^k_{t+1}$ has a millions of elements. It is not practical to send $\modelupdate^k_{t+1}$ with full precision for energy-constrained devices. Hence, we apply the same quantization scheme used in \acp{QNN} to $\modelupdate^k_{t+1}$ and denote its quantization result as $\qmodelupdate_{t+1}$. Then, $\schedulesize$ devices transmit their model update to the \ac{BS}. The received model updates are averaged by the \ac{BS}, and the next global model will be generated as below
\begin{align}
	\weights_{t+1} = \weights_t + \frac{1}{K}\sum_{k \in \scheduleset } \qmodelupdate_{t+1}.
\end{align}
The \ac{FL} system repeats this process until the global loss function converges to a target accuracy constraint $\accuracy$. We summarize the aforementioned algorithm in Algorithm 1.

Next, we propose the energy model for the computation and the transmission for our \ac{FL} system.

\begin{figure}[t!]
	\includegraphics[width=0.9\columnwidth]{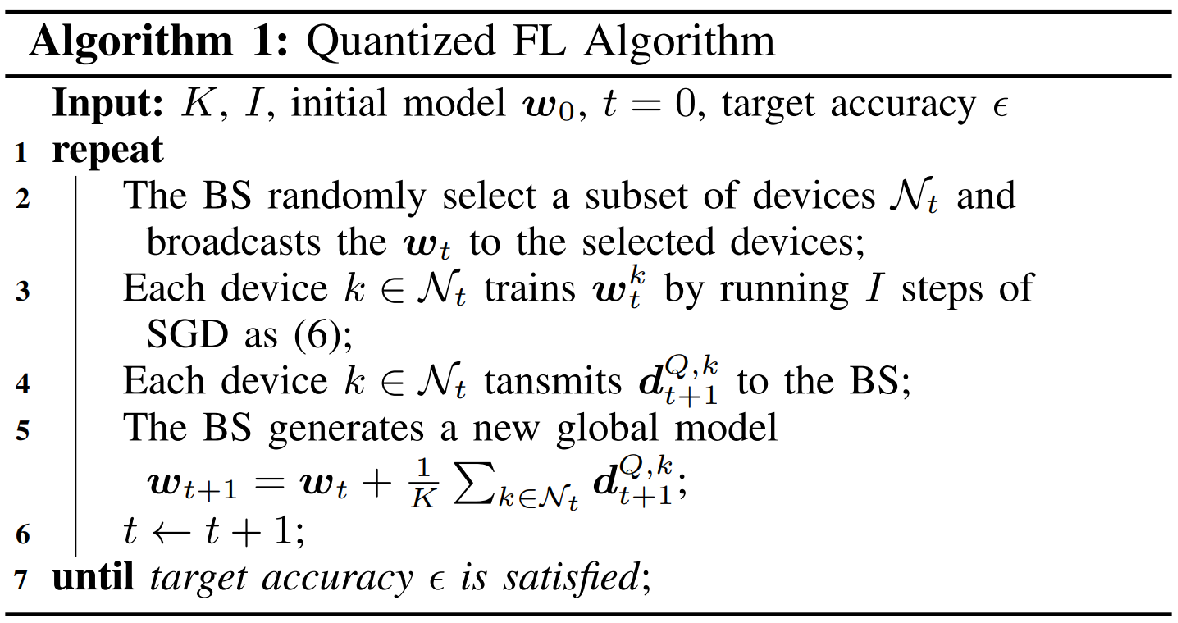}
	\captionsetup {singlelinecheck = false}
	
	\vspace{-5.0mm}	
\end{figure}
\subsection{Computing and Transmission model}
\subsubsection{Computing model}
	\begin{figure}
	\includegraphics[width=0.9\columnwidth]{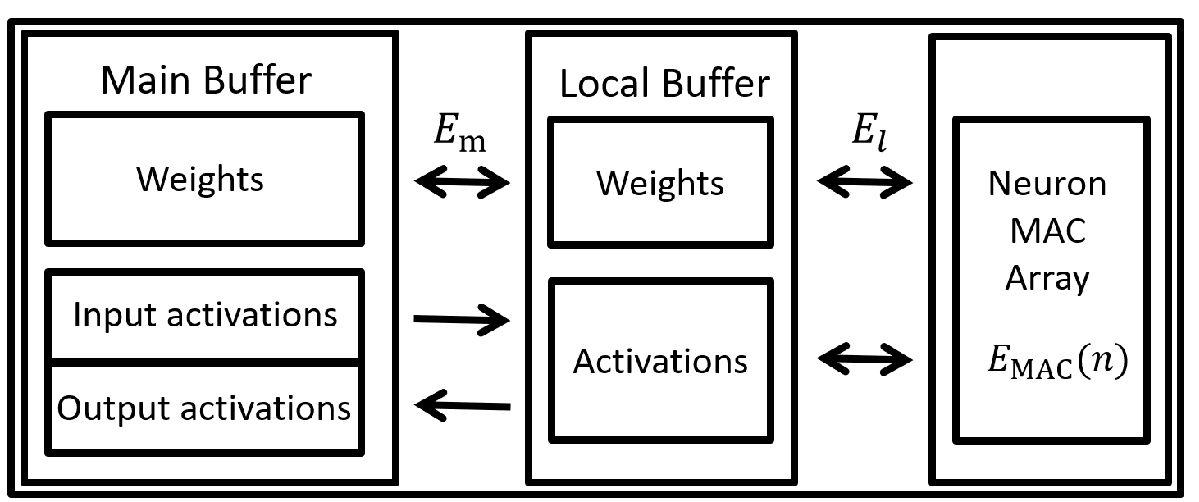}
	\captionsetup {singlelinecheck = false}

	\caption{An illustration of the two dimensional processing chip.}
	\label{fig:comptuing_model}
	
	\vspace{-6.0mm}
	
	\end{figure}
We consider a typical two dimensional processing chip for \acp{CNN} as shown in Fig. 2 \cite{MB:17}. This chip has a parallel neuron array, $\MACunits$ \ac{MAC} units, and two levels of memory: a main and a local buffer. A main buffer stores the current layers' weights and activations, while a local buffer caches currently used weights and activations. From \cite{MB:18},  we use the energy model of a \ac{MAC} operation for $\precision$ levels of precision $\MACE (\precision) = \MACamplitude \left( \precision /\premax \right) ^{\MACex}$, where $\MACamplitude > 0$, $1 < \MACex < 2$, and $\premax$ is the maximum precision level. Here, a \ac{MAC} operation includes the operation of a layer such as output calculation, batch normalization, activation, and weight update. Then, the energy consumption for accessing a local buffer $\localb$ can be modeled as $\MACE (\precision)$, and the energy for accessing a main buffer is $\mainb = 2 \MACE (\precision)$  \cite{MB:17}	.

The energy consumption of device $k$ for one iteration of local training is given by $\localE$ when $\precision$ bits are used for the precision level in the quantization. Then, $\localE$ is the sum of the computing energy $E_\text{C}(\precision)$, the access energy for fetching weights from the buffers  $E_\text{W}(\precision)$, and the access energy for fetching activations from the buffers $E_\text{A}(\precision)$, as follows \cite{MB:18}:
\begin{align}
	&\localE = E_\text{C}(\precision) + E_\text{W}(\precision) + E_\text{A}(\precision), \ka
	&E_{\text{C}}(\precision) = \MACE (\precision) \numMAC + 4\numout \ \MACE(\premax), \ka
	&E_\text{W}(\precision) = \mainb \numweights + \localb \numMAC \sqrt{{\precision}/{\MACunits \premax}}, \ka
	&E_\text{A}(\precision) = 2 \mainb \numout +  \localb \numMAC \sqrt{{\precision}/{\MACunits \premax}}, \label{computation_energy}
\end{align} 
where $\numMAC$ is the number of \ac{MAC} operations, $\numweights$ is the number of weights, and $\numout$ is the number of intermediate outputs throughout the network. For $E_{\text{C}}$, in a \ac{QNN}, batch normalization, activation function, gradient calculation, and weight update are done in full-precision $\premax$ to each output $\numout$ \cite{IH:16}. Once we fetch weights from a main to a local buffer, they can be reused in the local buffer afterward as shown in $E_\text{W}(\precision)$. In Fig. \ref{fig:comptuing_model}, a \ac{MAC} unit fetches weights from a local buffer for computation. Since we are using a two dimensional \ac{MAC} array of $\MACunits$ \ac{MAC} units, they can share fetched weights with the same row and column, which has $\sqrt{\MACunits}$ \ac{MAC} units respectively. In addition, a \ac{MAC} unit can fetch more weights due to the quantization with $\precision$ bits compared with when weights are represented in $\premax$ bits. Thus, we can reduce the access to a local buffer by the amount of $\sqrt{{\precision}/{\MACunits \premax}}$. A similar process applies to $E_\text{A}$ since activations are fetched (stored) from (to) the main buffer. 

\subsubsection{Transmission Model}
We use \ac{OFDMA} to transmit a model update to the \ac{BS}. The achievable rate of device $k$ is given by
\begin{align}
	\achievalble_k = B \log_2\left( 1 +\frac{P h_k}{N_0 B}  \right),
\end{align}
where $B$ is the allocated bandwidth, $h_k$ is the channel gain between device $k$ and the \ac{BS}, $P$ is the transmit power of each device, and $N_0$ is the power spectral density of white noise. After local training, device $k$ will transmit $\qmodelupdate_t$ to the \ac{BS} at given global iteration $t$. Then, the transmission time $T_k$ for uploading $\qmodelupdate_t$ is given by
\begin{align}
\displaystyle T_k (\precision) = \frac{||\qmodelupdate_t||}{\achievalble_k} = \frac{||\modelupdate^k_t|| \precision}{\achievalble_k \premax}.
\end{align}
Note that $\qmodelupdate_t$ is quantized with $\precision$ bits of precision while $\modelupdate^k_t$ is represented with $\premax$ bits. Then, the energy consumption for the uplink transmission is given by
\begin{align}
	\tE = T_k(\precision) \times P = \frac{P ||\modelupdate^k_t|| \precision} {B \log_2\left( 1 +\frac{P h_k}{N_0 B}  \right) \premax} . \label{transmission_energy}
\end{align}

In the following section, we formulate an energy minimizing problem based on the derived energy models.

\section{Proposed Approach for Energy-Efficient Federated \ac{QNN}} \label{sec:problem formulation}
We formulate an energy minimization problem while ensuring convergence under a target accuracy. A tradeoff exists between the energy consumption and the convergence rate with respect to $\precision$. Hence, finding the optimal $n$ is important to balance the tradeoff and achieve the target accuracy. We propose a numerical method to solve this problem.

We aim to minimize the expected total energy consumption until convergence under the target accuracy as follows:
\begin{subequations}
\begin{align}
		\min_{\precision} \quad & \E \left[ \sum_{t=1}^{T} \sum_{k\in\scheduleset} \tE + \SGDrun \localE \right] \label{obj} \\
		\textrm{s.t.} \quad & \precision \in [1, \dots, \premax],\\
		& \E[\globall(\weights_T)] - \globall(\weights^*) \leq \accuracy \label{accuracy_constraint},  
\end{align}
\end{subequations}
where $\SGDrun$ is the number of local iterations, $\E[\globall(\weights_T)]$ is the expectation of global loss function after $T$ global iteration, $\globall(\weights^*)$ is the minimum value of $\globall$, and $\accuracy$ is the target accuracy.  

Since $\schedulesize$ devices are randomly selected at each global iteration, we can derive the expectation of the objective function of \eqref{obj} as follows
\begin{align}
\obj &= \E \left[ \sum_{t=1}^{T} \sum_{k\in\scheduleset} \tE + \SGDrun \localE \right] \ka
&= \frac{\schedulesize T}{N} \sum_{k=1}^{N} \left\{ \tE + \SGDrun \localE \right\}.
\end{align}
To represent $T$ with respect to $\accuracy$, we assume that the loss function is $\smooth$-smooth, $\strong$-strongly convex and that the variance and the squared norm of the stochastic gradient are bounded by $\sgdvariance_k^2$ and $\sgdbound$ for device $k$, $\forall k \in \Whole$, respectively. Before we present the expression of $T$, in the following lemma, we will first analyze the quantization error of the stochastic quantization in Sec. \ref{sec:system model}.
\begin{lemma}
	For the stochastic quantization $\Quantize(\cdot)$, a scalar value $w$, and a vector $\weights \in \mathbb{R}^d$, we have \label{Lem1}
	\begin{align}
		&\E[\Quantize(w)] = w, \quad \E[(\Quantize(w) - w)^2] \leq \frac{1}{2^{2\precision}}, \\
		&\E[\Quantize(\weights)] = \weights, \quad \E[||\Quantize(\weights) - \weights||^2] \leq \frac{d}{2^{2\precision}}.
	\end{align}
\end{lemma}
\begin{proof}
	We first derive $\E[\Quantize(w)]$ as
	\begin{align}
		\E[\Quantize(w)] &= \lfloor w \rfloor \frac{\lfloor w \rfloor \hspace{-0.5mm} + \hspace{-0.5mm} \smallest \hspace{-0.5mm} - \hspace{-0.5mm} w}{\smallest} \hspace{-0.5mm} + \hspace{-0.5mm} (\lfloor w \rfloor \hspace{-0.5mm} +  \hspace{-0.5mm} \smallest) \frac{w \hspace{-0.5mm} - \hspace{-0.5mm} \lfloor w \rfloor}{\smallest} \hspace{-0.5mm}  = \hspace{-0.5mm} w . \label{Lem1-1}
	\end{align} 
	Similarly,  $\E[(\Quantize(w)-w)^2]$ can be obtained as
	\begin{align}
		\E \hspace{-0.3mm} [ \hspace{-0.3mm} (\Quantize( \hspace{-0.5mm} w \hspace{-0.5mm}) \hspace{-0.7mm} - \hspace{-0.7mm}w)^2 \hspace{-0.2mm}] & \hspace{-0.5mm}=  \hspace{-0.5mm} ( \hspace{-0.3mm} \lfloor \hspace{-0.5mm} w \hspace{-0.5mm} \rfloor \hspace{-0.7mm} -  \hspace{-0.7mm} w)^2 \frac{\lfloor \hspace{-0.5mm} w \hspace{-0.5mm} \rfloor  \hspace{-0.7mm} +  \hspace{-0.7mm} \smallest \hspace{-0.5mm}  - \hspace{-0.5mm}	 w}{\smallest} \hspace{-0.5mm} + \hspace{-0.5mm} (\lfloor\hspace{-0.5mm} w \hspace{-0.5mm} \rfloor \hspace{-0.5mm} + \hspace{-0.5mm} \smallest \hspace{-0.5mm} - \hspace{-0.5mm}w)^2 \frac{w \hspace{-0.5mm} - \hspace{-0.5mm} \lfloor \hspace{-0.5mm} w \hspace{-0.5mm} \rfloor}{\smallest} \ka
		&= (w - \lfloor \hspace{-0.5mm} w \hspace{-0.5mm} \rfloor) (\lfloor w \rfloor+\smallest -w) \ka
		&\leq \frac{\smallest^2}{4} = \frac{1}{2^{2\precision}},  \label{AMGM}
	\end{align}
	where \eqref{AMGM} follows from the arithmetic mean and geometric mean inequality. Since expectation is a linear operator, we have $\E[\Quantize(\weights)] = \weights$ from \eqref{Lem1-1}. From the definition of the square norm, $\E[||\Quantize(\weights) - \weights||^2]$ can obtained as
	\vspace{-1.5mm}
	\begin{align}
		\E[||\Quantize(\weights) - \weights||^2] = \sum_{j=1}^{d} \E[(\Quantize(w_j) - w_j)^2] \leq \frac{d}{2^{2\precision}}.
	\end{align}
\vspace{-3.0mm}
\end{proof}
From Lemma \ref{Lem1}, we can see that our quantization scheme is unbiased as its expectation is zero. However, the quantization error can still increase for a large model. 
We next leverage the results of Lemma \ref{Lem1}, \cite{SZ:21},  and\cite{Xi:20} so as to derive $T$ with respect to $\accuracy$ in the following proposition. 
\begin{prop}
\vspace{-1.0mm} For learning rate $\learningrate_t = \frac{\beta}{t+\gamma}, \smooth < \frac{2\learningrate_t}{2\learningrate_t^2+1}, \beta > \frac{1}{\strong}, \ \text{and} \ \gamma > 0$, we have 
\begin{align}
	\E[\globall(\weights_T) - \globall(\weights^*)] \leq \frac{\smooth}{2}\frac{\constant}{T+\gamma}, \label{T_accuracy}
\end{align}

where $\constant$ is 
\vspace{-0.5mm}
\begin{align}
	\constant \hspace{-0.5mm} &=  \hspace{-0.5mm} \sum_{k=1}^{\Whole} \hspace{-0.5mm} \frac{\sgdvariance_k ^2}{\Whole^2} \hspace{-0.7mm} + \hspace{-0.7mm} \frac{d}{2^{2\precision}} \hspace{-1.0mm} \left( \hspace{-0.5mm} 1 \hspace{-0.5mm} + \hspace{-0.5mm} \frac{2 \SGDrun \sgdbound^2}{\schedulesize} \hspace{-0.5mm} \right) \hspace{-0.7mm} + \hspace{-0.7mm} 4(\SGDrun \hspace{-0.7mm}- \hspace{-0.7mm}1)^2 \sgdbound^2  \hspace{-0.7mm}+\hspace{-0.7mm} \frac{4(\Whole \hspace{-0.7mm} - \hspace{-0.7mm}\schedulesize)}{\schedulesize (\Whole \hspace{-0.7mm} - \hspace{-0.7mm} 1)} \SGDrun^2 \sgdbound^2. \label{v}
\end{align} \label{Prop}
\end{prop}
\vspace{-8.0mm}
\begin{proof}
	The complete proof is omitted due to space limitations. Essentially, proposition \ref{Prop} can be proven by using Lemma \ref{Lem1}, the convergence result with a quantized model update \cite{SZ:21}, and replacing the \ac{SGD} weight update in \cite{Xi:20} with \eqref{SGD_update}. 
\end{proof}
From Proposition \ref{Prop}, We let \eqref{T_accuracy} be upper bounded by $\accuracy$ in \eqref{accuracy_constraint} as follows
\begin{align}
	\E[\globall(\weights_T) - \globall(\weights^*)] \leq \frac{\smooth}{2}\frac{\constant}{T+\gamma} \leq \accuracy. \label{accuracy_inequality}
\end{align}
We then take equality in \eqref{accuracy_inequality} to obtain $T = \smooth \constant/(2 \accuracy) - \gamma$ and approximate the problem as
\begin{subequations}
\begin{align}
		\min_{\precision} \quad & \frac{\schedulesize}{\Whole} \left( \frac{\smooth \constant}{2 \accuracy} \hspace{-0.5mm} -  \hspace{-0.5mm} \gamma \right) \sum_{k=1}^{\Whole} \left\{ \tE \hspace{-0.5mm}+\hspace{-0.5mm} \SGDrun \localE \right\}  \hspace{-0.5mm}=  \hspace{-0.5mm}f_E (\precision) \label{obj_approximation} \\
		\textrm{s.t.} \quad & \precision \in [1, \dots, \premax].
\end{align}
\end{subequations}
Note that any optimal solution $\precision^*$ from problem \eqref{obj_approximation} can satisfy problem \eqref{obj} \cite{BL:21}. For any feasible $T$ from \eqref{accuracy_inequality}, we can always choose $T_0 > T$ such that $T_0$ satisfies \eqref{accuracy_constraint}. 

Now, we relax $\precision$ as a continuous variable, which will be rounded back to an integer value. 
From \eqref{computation_energy}, \eqref{transmission_energy}, and \eqref{v}, we can observe that, as $\precision$ increases, $\tE$ and $\localE$ becomes larger while $T$ decreases. Hence, we can know that a local optimal $\precision^*$ may exist for minimizing $\obj$. Since $\obj$ is differentiable with respect to $\precision$ in the given range, we can find $\precision^*$  by solving $\partial{\obj}/ {\partial n} =0$ from Fermat's Theorem \cite{Bauhe:11}. Although it is difficult to derive $\precision^*$ analytically, we can obtain it numerically using a line search method. Hence, we can find a local optimal solution, which minimizes the total energy consumption under the given target accuracy.  
 
\section{Simulation Results} \label{sec:simlulation results}
For our simulations, we uniformly deploy $\Whole = 50$ devices over a square area of size $100$ m $\times$ $100$ m serviced by one \ac{BS} at the center, and we assume a Rayleigh fading channel with a path loss exponent of 2. We also use MNIST dataset. Unless stated otherwise, we use $P = 100$ mW, $B = 10$ MHz, $N_0 = -100$ dBm, $\schedulesize = 30$, $\SGDrun= 5$, $\premax =32$ bits, $\accuracy = 0.01$, $\beta = 5$, $\smooth = 1$, $\mu = 1$, $\gamma = 1$, $\sgdvariance_k = 1$, and $\sgdbound = \sqrt{4\smooth\epsilon}$, $\forall k = 1, \dots, N$ \cite{LM:18}. For the computing model, we set $\MACamplitude = 3.7$ pJ and $\MACex = 1.25$ as done in \cite{MB:18}, and we assume that each device has the same architecture of the processing chip. All statistical results are averaged over $10000$ independent runs

\begin{figure}
	\includegraphics[width=0.85\columnwidth]{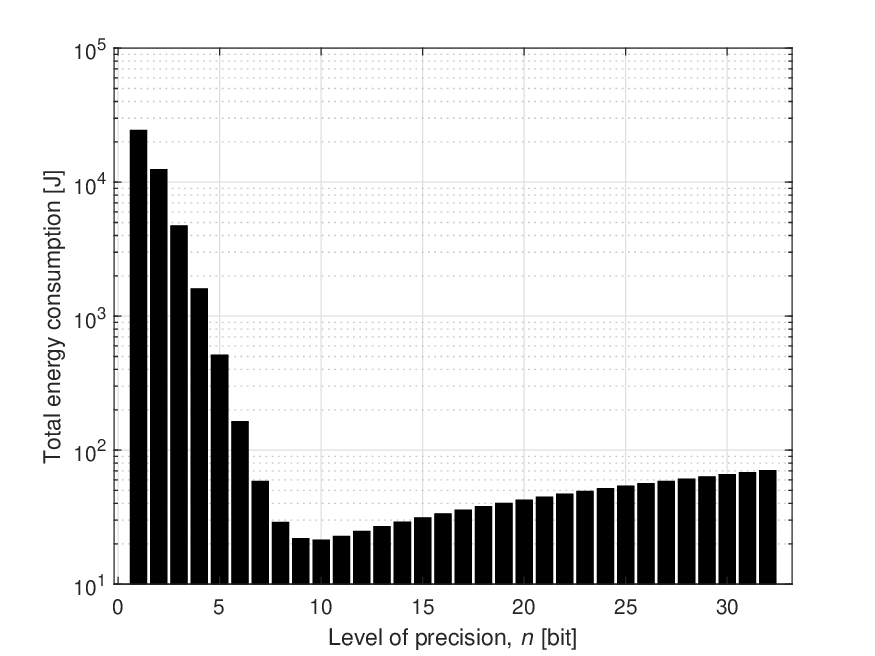}
	\captionsetup {singlelinecheck = false}

	\caption{Total energy consumption for varying level of precision.}
	\label{fig:energy_consumption}
	\vspace{-4.0mm}
\end{figure}
	
Figure \ref{fig:energy_consumption} shows the total energy consumption of the \ac{FL} system until convergence for varying levels of precision $\precision$. In Fig. \ref{fig:energy_consumption}, we assume a \ac{QNN} structure with two convolutional layers: 32 kernels of size $3\times3$ with one padding and three of strides and 32 kernels of size $3\times3$ with one padding and two of strides, each followed by $2\times2$ max pooling. Then, we have one dense layer of 220 neurons and one fully-connected layer. In this setting, we have $\numMAC = 20.64\times10^6, \numweights = 0.18\times10^6,$ and $\numout = 1354$.  From this figure, we can see that the total energy consumption decreases and then increases with $\precision$. This is because when $\precision$ is small, quantization error becomes large as shown in Lemma \ref{Lem1}, which slows down the convergence rate in \eqref{T_accuracy}. However, as $\precision$ increases, the energy consumption for the local training and transmission also increases. Hence, a very small or very high $n$ may induce undesired large quantization error or unnecessary energy consumption due to a high level of precision. From this figure, we can see that $\precision = 10$ can be optimal for minimizing energy consumption for our system.

\begin{figure}
	\includegraphics[width=0.85\columnwidth]{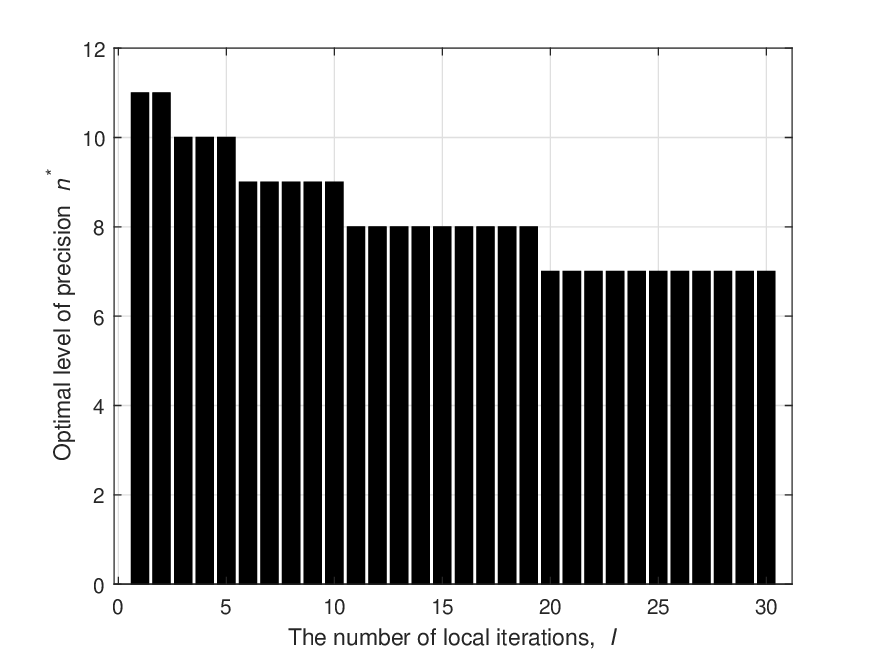}
	\captionsetup {singlelinecheck = false}

	\caption{Optimal level of precision for varying the number of local iterations.}
	\label{fig:Iteration_number}
	\vspace{-6.0mm}
\end{figure}

Figure \ref{fig:Iteration_number} shows the optimal level of precision $\precision^*$ when varying the number of local iterations $\SGDrun$. We use the same \ac{CNN} architecture in Fig. \ref{fig:energy_consumption}. We can observe that $\precision^*$ increases with $\SGDrun$. This is because, as $\SGDrun$ increases, the local models converge to the local optimal faster as \ac{SGD} averages out the effect of quantization error \cite{IH:16}. Hence, a lower $\precision$ can be chosen by leveraging the increased $\SGDrun$ to minimize the total energy consumption. We can observe that only $\precision = 7$ is required at $\SGDrun = 20$ while we need $\precision = 10$ at $\SGDrun = 3$. 
	
\begin{figure}
	\vspace{-0.0mm}
	\includegraphics[width=0.85\columnwidth]{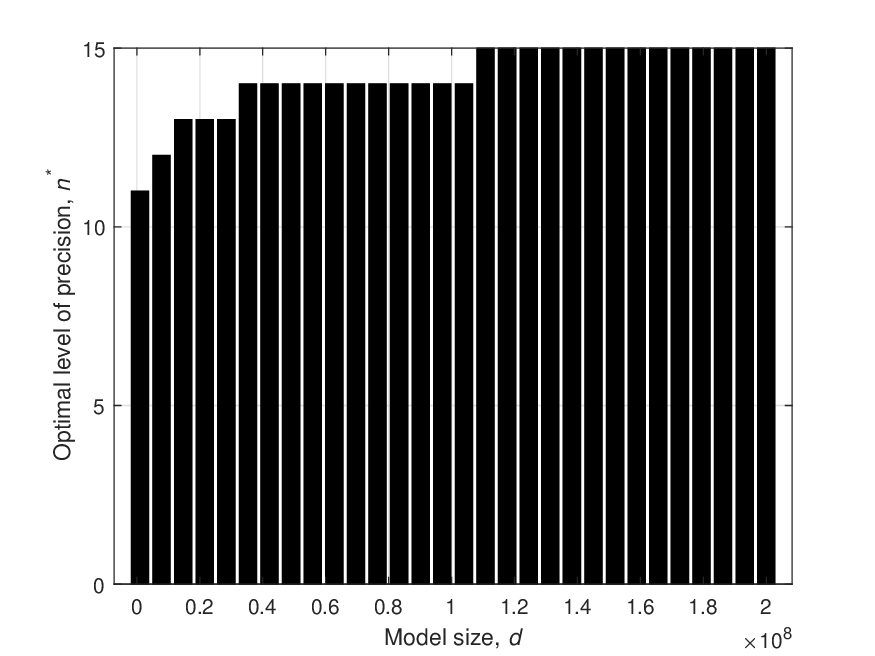}
	\captionsetup {singlelinecheck = false}

	\caption{Optimal level of precision for varying model size.}
	\label{fig:model_size}
	\vspace{-6.0mm}
\end{figure}

Figure \ref{fig:model_size} presents the optimal level of precision $\precision^*$ for varying model size $d$. Note that $d$ equals to the number of model parameters $\numweights$. To scale the number of \ac{MAC} operations for increasing $d$ accordingly, we set $\numMAC = 0.5 \times 10^3 \numweights$. From Fig. \ref{fig:model_size}, we can see that $\precision^*$ increases with $d$. From Lemma \ref{Lem1}, the quantization error accumulates as $d$ increases. This directly affects the convergence rate in \eqref{T_accuracy} resulting in both increased global iterations and the total energy consumption. Therefore, to mitigate the increasing quantization error from increasing $d$, a larger level of precision may be chosen.

\begin{figure}
		\vspace{-1.0mm}
	\includegraphics[width=0.85\columnwidth]{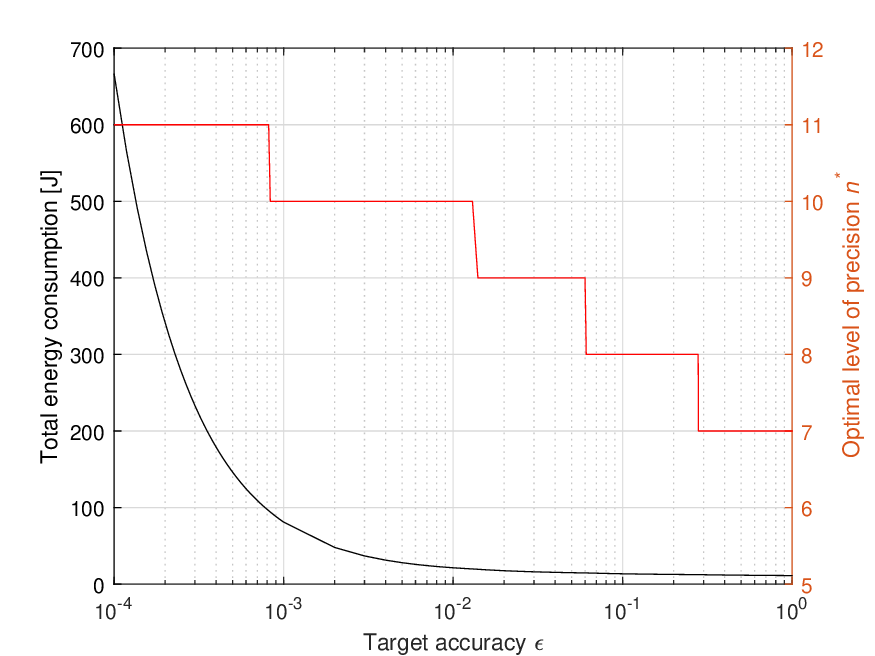}
	\captionsetup {singlelinecheck = false}

	\caption{Total energy consumption and optimal level of precision for varying target accuracy.}
	\label{fig:energy_accuracy_precision}
	\vspace{-5.5mm}
\end{figure}

Figure \ref{fig:energy_accuracy_precision} shows the total energy consumption and $\precision^*$ when varying a target accuracy $\accuracy$. For these results, we use the same \ac{CNN} architecture as Fig. \ref{fig:energy_consumption}. We can see that a higher accuracy level requires larger total energy consumption and more bits for data representation to mitigate the quantization error. In addition, the \ac{FL} system needs a more number of global iterations to achieve $\epsilon$ from \eqref{accuracy_inequality}. As $\epsilon$ becomes looser, a lower $\precision$ can be chosen. From Fig. \ref{fig:energy_accuracy_precision}, we can see that an additional $127$ J energy is required to increase  $\accuracy$ from 0.01 to $0.001$ while one additional bit of precision is needed.

\begin{figure}
			\vspace{-0.0mm}
	\includegraphics[width=0.85\columnwidth]{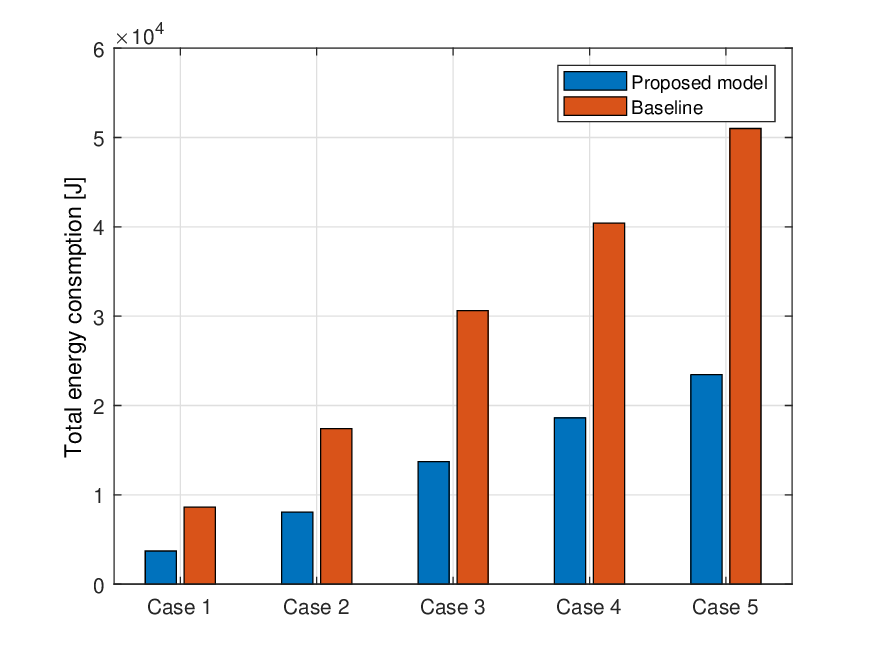}
	\captionsetup {singlelinecheck = false}
	
		\vspace{-3.0mm}
	
	\caption{Total energy consumption for various \ac{CNN} models.}
	\label{fig:various}
	\vspace{-6.0mm}
\end{figure}

Figure \ref{fig:various} compares the total energy consumption until convergence for varying the size of \ac{CNN} models with the baseline that uses standard 32 bits for data representation. Case 1 is assumed to be CNN of two layers and have $d = 25.6\times10^6$ and $N_c = 1.8\times10^9$. Case 2 and 3 are assumed to be CNN of five layers. They have  $61.6\times10^6$ and  $83.5\times10^6$ number of parameters and  $0.35\times10^9$ and $83.5\times10^6$ number of MAC operations, respectively. Case 4 is 7 layers of CNN with  $d = 115.1\times10^6$ and $N_c = 10.4\times10^9$. Lastly, we assume Case 5 is 9 layers of CNN with $d =138.6\times10^6$ and $15.5\times10^9$. The corresponding $\precision^*$ are 13, 14, 14 15, 15, respectively. We can see that our \ac{FL} scheme is more effective for \ac{CNN} models with a large model size. Note that Case 5 has 138.8 M weights while Case 1 has 25.6 M weights. In particular, for Case 5, we can reduce the total energy consumption up to 53$\%$ compared to the baseline.

\vspace{-2.0mm}
\section{Conclusion} \label{sec:conclusion}
\vspace{-0.0mm}
In this paper, we have studied the problem of energy-efficient quantized \ac{FL} over wireless networks. We have presented the energy model for the quantized \ac{FL} based on the physical structure of a processing chip and the convergence rate. Then, we have formulated an energy minimization problem that considers a level of precision in the quantized \ac{FL}. To solve this problem, we have used a line search method. Simulation results have shown that our model requires much less energy than a standard \ac{FL} model for convergence. The results particularly show significant improvements when the local models rely on large neural networks. In essence, this work provides the first holistic study of the tradeoff between energy, precision, and accuracy for FL over wireless networks.
\vspace{-6.0mm}

\bibliographystyle{IEEEtran}
\bibliography{Bibtex/StringDefinitions,Bibtex/IEEEabrv,Bibtex/mybib}

\end{document}